\documentclass[runningheads]{llncs}

\usepackage[mobile]{eccv}

\usepackage{eccvabbrv}
\usepackage{subcaption}
\usepackage{graphicx}
\usepackage{booktabs}
\usepackage{caption,subcaption}
\usepackage[accsupp]{axessibility}  
\usepackage{multirow}
\usepackage{array}
\usepackage{bm}
\usepackage{algpseudocode}
\usepackage[dvipsnames]{xcolor}
\usepackage{arydshln}
\usepackage{nicematrix}

\usepackage{algorithm}
\algnewcommand{\LineComment}[1]{\State \(\triangleright\) #1}
\newcolumntype{C}[1]{>{\centering\let\newline\\\arraybackslash\hspace{0pt}}m{#1}}

\usepackage[pagebackref,breaklinks,colorlinks,citecolor=eccvblue]{hyperref}

\usepackage{orcidlink}

\begin{document}

\title{Fast and Simple Explainability for Point Cloud Networks}


\author{Meir Yossef Levi\inst{1} \and
Guy Gilboa\inst{1}}


\institute{Technion - Israel Institute of Technology, Haifa, Israel \\
\email{me.levi@campus.technion.ac.il}\\
\email{guy.gilboa@ee.technion.ac.il}}

\maketitle

\begin{abstract}
We propose a fast and simple explainable AI (XAI) method for point cloud data. It computes pointwise importance with respect to a trained network downstream task. This allows better understanding of the network properties, which is imperative for safety-critical applications. 
In addition to debugging and visualization, our low computational complexity facilitates online feedback to the network at inference. This can be used to reduce uncertainty and to increase robustness. 
In this work, we introduce \emph{Feature Based Interpretability} (FBI), where we compute the features' norm, per point, before the bottleneck. We analyze the use of gradients and post- and pre-bottleneck strategies, showing pre-bottleneck is preferred, in terms of smoothness and ranking. 
We obtain at least three orders of magnitude speedup, compared to current XAI methods, thus, scalable for big point clouds or large-scale architectures. Our approach achieves SOTA results, in terms of classification explainability.
We demonstrate how the proposed measure is helpful in analyzing and characterizing various aspects of 3D learning, such as rotation invariance, robustness to out-of-distribution (OOD) outliers or domain shift and dataset bias. 
\end{abstract}

\section{Introduction}
\label{sec:intro}

Ranking the importance of points within a point cloud is fundamental for gaining deeper understanding and for improving the network's performance in various tasks. Being able to compute importance fast, without resorting to gradient computations, can be of great advantage. It facilitates the use at inference, providing additional capabilities for the network. Current XAI methods for point clouds are slow since they either compute gradients or are based on time-consuming iterative processes.
In this work, we analyze the use of gradients for determining importance in graph neural-networks. We show that the common pooling bottleneck architecture, and specifically Max-Pooling, introduces challenges for gradient-based methods. Importance becomes non-smooth, with either extreme values of flat areas, such that high quality ranking is difficult to obtain. The same phenomenon occurs for post-bottleneck measures, such as critical points \cite{pointnet}. 
We thus opt for a pre-bottleneck computation and observe that the $L^1$ norm
of the features (per-point) is a reliable indicator of influence.
We show qualitatively and quantitatively that high quality ranking of importance for the downstream task is obtained. Several examples demonstrate  how this measure can be used to examine rotation invariance, robustness to outliers and resilience to domain shifts.

\section{Related Work}

\textbf{3D classification.}
Several point cloud classification networks have been proposed in recent years \cite{dgcnn, curvenet, cloud_walker, pointnet, pointnet++, point_mlp, pointbert, gdanet, paconv, pct}. PointNet \cite{pointnet} pioneered the application of learning on raw 3D coordinates and introduced the concept of \textit{Critical Points} as the set of active points after the last pooling layer. This approach motivated subsequent architectures that embrace the 3D Euclidean space. For instance, Dynamic Graph CNN (DGCNN) \cite{dgcnn} introduced EdgeConv, a learnable layer combining local and global information. By iteratively reproducing the graph based on learnt features, the network better learns the semantics of the shape. Geometry-Disentangled Network (GDANet) \cite{gdanet} dynamically decomposes a shape into contour and flat areas for improved understanding.

\textbf{Robust classification.} Robust Point-Cloud Classifier (RPC) \cite{modelnet_c} combines the most robust modules in a typical classification network. PointGuard \cite{pointguard} and PointCert \cite{point_cert} propose a provable scheme for classifying noisy samples, splitting the classification process into distinct random samples and combining them using majority voting. Ensemble of Partial Point-Clouds (EPiC) \cite{epic} advocates using a diverse set of subsamples, encompassing Random, Patches, and Curves.
PointCleanNet \cite{pointcleannet} and PointASNL \cite{point_asnl} suggest learnable approaches for outlier filtering.

\textbf{Self-Supervised methods.} Occlusion Completion (OcCo) \cite{occo} suggests learning semantic correlations in a 3D shape by training on the completion of hidden parts from a certain camera view, and CrossPoint \cite{cross_point} learns semantics through contrastive learning on correspondences between point clouds and images.

\textbf{Rotation-invariant networks.} A fundamental requirement for 3D classification networks is the ability to correctly classify shapes  under rotation. Local-Global-Representation (LGR-net) \cite{lgrnet} encodes rotation-invariant global and local features, building a rotation-invariant network. Our experiments show the network is indeed highly robust to rotations but is still slightly affected by them.

\textbf{Explainable AI (XAI).} The goal of XAI is to expose the rationale of the network, usually by highlighting the regions of the input that most affected the network's output. For image data, numerous approaches have been proposed, such as \cite{gradcam, axiomatic, lrp, lime}. For instance, Gradient-weighted Class Activation Mapping (GradCam) \cite{gradcam} uses the gradients of any target flowing into the final convolution to produce explanations. IntegratedGradients \cite{axiomatic} leverages standard gradient calculation by integrating it along the path from a null input (e.g. black image) to the original image. Local Interpretable Model-agnostic Explanations (Lime) \cite{lime} suggests explaining a complex network with a simpler one through sampling and minimizing an adequate loss function. \emph{Point cloud explainability} is a less explored area of research. Point-Cloud Saliency Maps \cite{saliency_maps} proposes to slide points toward the center of mass to estimate their influence, considering this region as non-influential. Point-Lime \cite{point_lime} is an adaptation of Lime \cite{lime} in the 3D field. Both approaches are slow, requiring an iterative process. Worth mentioning is PointHop \cite{point_hop}, a dedicated, learnable network for explainability. However, in our research, we focus on explaining existing networks, rather than developing new interpretable ones. In the following section we set the notations, present our explainability measure and provide rationale for using the pre-bottleneck part of the network. 

\section{Method}

Consider a point cloud, \(X\), comprising \(N\) points in 3D space: \(\{X_1, \ldots, X_N\}\), \(i \in [N]\), \(X_i \in \mathbb{R}^D\), where in 3D coordinates $D=3$ (in general the input may be of higher dimensions). 
Let \(X_F \in \mathbb{R}^{N \times F} \) denote the per-point feature vector, where \(X_F(i,\cdot)\) is a vector of \(F\) real-valued features of the point \(X_i\). To obtain a global feature vector in a permutation-invariant manner, commonly used techniques include Max-Pooling or its combination with Mean-Pooling. The pooling is performed with respect to the points dimension, so following the pooling bottleneck we have \(X_G = \underset{N}{Pooling}(X_F) \in \mathbb{R}^F\).

\begin{figure*}[ptbh!]
  \centering
   \includegraphics[width = 0.8\linewidth, height =0.168 
   \linewidth]{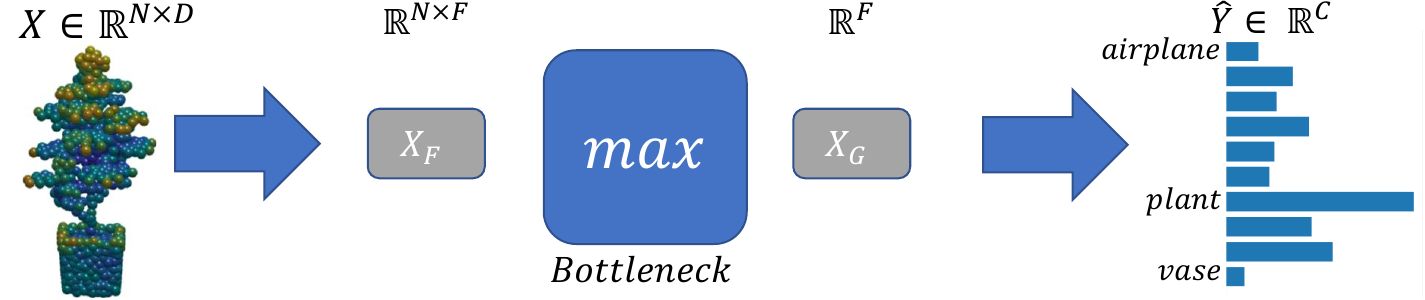}
   \caption{{\bf Typical data flow of a point cloud classification architecture.}}
   \label{fig:scheme}
\end{figure*}

\subsection{Feature-Based Interpretability (FBI)}
Our method is based on the intermediate features of the network probed from the pre-bottleneck stage of the network. A schematic representation of the typical data flow in a point cloud classification architecture is illustrated in \cref{fig:scheme}.
We show there is a strong correlation between the magnitude of the features, the importance of their semantic meaning, and consequently their contribution to the network's downstream task.

\begin{definition}[FBI]
The FBI measure of a point $X_i$ is defined by
\begin{equation}
FBI(i) := \sum_{k=1}^F |X_F(i,k)| .   
\end{equation} 
\end{definition}
We first give the rationale why this measure accounts well for points influence. We then support our claims qualitatively and quantitatively. 

\subsubsection{Probing prior to bottleneck}
The bottlenecks of graph neural networks are often highly aggressive and reduce significant information regarding the input data. For the case of Max-Pooling, let us examine the gradient of the network prediction with respect to a data point.
Let ${\bar{X}_F} \in \mathbb{R}^{N \cdot F}$ be the column stack of the matrix ${X_F} \in \mathbb{R}^{N \times F}$. The derivative of the prediction, $\hat{Y} \in \mathbb{R}^C$, with respect to a point $X_i$, using the chain rule, is:
\begin{equation*}
    \frac{\partial \hat{Y}}{\partial X_i} =
    \underbrace{\overbrace{\underbrace{\frac{\partial \hat{Y}}{\partial X_G}}_{C \times F}}^\text{Post-Bottleneck} \cdot
    \overbrace{\underbrace{\frac{\partial X_G}{\partial \bar{X}_F}}_{F \times (N \cdot F)}}^\text{Bottleneck} \cdot
    \overbrace{\underbrace{\frac{\partial \bar{X}_F}{\partial X_i}}_{(N \cdot F) \times D}}^\text{Pre-Bottleneck}}_{C \times D}
\end{equation*}
We assume Max-Pooling, so  $X_G = \underset{N}{\max}(X_F)$. Recall the derivative of the max function is 1 at the maximal value and zero for all other entries. Thus, the explicit term of 
$\frac{\partial X_G(k)}{\partial X_F} $
is a matrix unit $\mathbb{E}_{j_k,k}$ that has a single nonzero entry with value 1 at $(j_k, k)$, where $j_k \in \{1, \cdots ,N\}$ is the index of the point with maximal value corresponding to feature $k$ (that is, $j_k : X_F(j_k,k) > X_F(j,k), \forall j \neq j_k$). 
Thus, for the column stacked matrix $\bar{X}_F$, let us denote $\frac{\partial X_G(k)}{\partial \bar{X}_F}  \in \mathbb{R}^{N \cdot F}$, which is a transposed column stacked $\mathbb{E}_{j_k,k}$, as $\bm{\delta} _{(k - 1) \cdot N + j_k}$, and in general,
\begin{equation*}
    \frac{\partial X_G}{\partial \bar{X}_F} = 
    \begin{pNiceArray}{w{c}{20mm}}[margin=1pt]
        \bm{\delta}_{j_1} \\
        \bm{\delta} _{N + j_2} \\
        \vdots\\
        \bm{\delta} _{(F - 1) \cdot N + j_F} \\
    \end{pNiceArray}
    \in \mathbb{R}^{F \times N\cdot F} .
\end{equation*}

\begin{proposition}[Existence of zero gradients]
Assume $\frac{\partial X_F(j,\cdot)}{\partial X_i} = 0$, $\forall i, j \in \{1, \ldots, N\}$, $j \neq i$ (i.e., PointNet), and $N > F$. Then, there exist at least $N-F$ points such that $\frac{\partial\hat{Y}}{\partial X_i} = 0$.
\label{lemma:zero_gradient}
\end{proposition}

\begin{proof}
$\frac{\partial X_G}{\partial X_i} = \frac{\partial X_G}{\partial \bar{X}_F} \cdot \frac{\partial \bar{X}_F}{\partial X_i}$.
Plugging-in the explicit term of $\frac{\partial X_G}{\partial \bar{X}_F}$,
\[
\frac{\partial X_G}{\partial X_i} =
\begin{pNiceArray}{w{c}{20mm}}[margin=1pt]
    \bm{\delta}_{j_1} \\
    \bm{\delta}_{N + j_2} \\
    \vdots \\
    \bm{\delta}_{(F - 1) \cdot N + j_F} \\
\end{pNiceArray}
\cdot \frac{\partial \bar{X}_F}{\partial X_i}
\]
we get,
\begin{equation*}
    \frac{\partial X_G}{\partial X_i} =
    \begin{pNiceArray}{w{c}{26mm}}[margin=1pt]
        \frac{\partial \bar{X}_F(j_1)}{\partial X_i} \\
        \frac{\partial \bar{X}_F(N + j_2)}{\partial X_i} \\
        \vdots \\
        \frac{\partial \bar{X}_F((F - 1) \cdot N + j_F)}{\partial X_i} \\
    \end{pNiceArray}
    =
    \begin{pNiceArray}{w{c}{20mm}}[margin=1pt]
        \frac{\partial X_F(j_1, 1)}{\partial X_i} \\
        \frac{\partial X_F(j_2, 2)}{\partial X_i} \\
        \vdots \\
        \frac{\partial X_F(j_F, F)}{\partial X_i} \\
    \end{pNiceArray}
\end{equation*}
Using the first assumption 
yields $\frac{\partial X_G}{\partial X_i} = 0$ 
for any $i \notin \{j_1, j_2, \ldots, j_F\}$,
and thus $\frac{\partial \hat{Y}}{\partial X_i} = 0$ for these indices. Since the set $\{j_1, j_2, \ldots, j_F\}$ contains at most $F$ elements and  $N > F$, there exist at least $N-F$ elements for which $\frac{\partial \hat{Y}}{\partial X_i} = 0$.
\end{proof}

In \cref{fig:intro} we visualize the gradients computed on an airplane sample using PointNet \cite{pointnet} and DGCNN \cite{dgcnn}. Clearly, there exist points that the gradients are equally zero when applied on PointNet \cite{pointnet}, whereas in DGCNN \cite{dgcnn} we observe similar relaxed trend. Points in less discriminative regions, i.e, wing base (outside the critical set), have relatively low gradients.

\subsection{Critical points analysis}
We examine in more detail the case of \textit{Critical Points} (CP), a method commonly employed for probing after pooling. Qualitative comparison between FBI and CP is illustrated in \cref{fig:fbi_critical}.  
In \cite{pointnet} \textit{Critical Points} were defined as the points that remain active after the last Max-Pooling layer. That is,
\begin{equation}
\label{eq:critical}
CP(i) := 
\begin{cases}
    1,              & \text{if } \exists k \text{ s.t. } X_f(i,k)>X_f(j,k) \text{, } \forall j\neq i\\
    0,              & \text{otherwise}
\end{cases}
\end{equation}
The \emph{Critical Set} is defined by,
\[ S_C := \{i\,:\, CP(i) = 1\}. \]

\begin{proposition}[Smoothness]
Assume the K-nearest-neighbors (KNN) graph of $X$ is a connected graph. Let $h$ be a positive constant such that $\max |X_i - X_j| \le h$, $\forall i \in \{1, \cdots ,N\}$, $\forall X_j \in \text{KNN}(X_i)$. Assume $\frac{\partial X_F(j,\cdot)}{\partial X_i} = 0$ (i.e., PointNet), and $N > F$. Then, the influence induced by Critical Points is K-Lipschitz with $K \ge \frac{1}{h}$. 
\label{lemma:smoothness}
\end{proposition}

\begin{proof}
Using \cref{lemma:zero_gradient}, there exist at least $N-F$ points outside the critical set, and since $F>0$, there exist at least a single point in the critical set,  $S_c, \bar{S}_c \notin \emptyset$. Therefore, for a connected graph, $\exists i,j:\{CP(X_i) = 0,\,CP(X_j) = 1\}$ such that $X_j \in \text{KNN}(X_i)$ (See Auxiliary proof in the supp.). The Lipschitz condition for $CP$ is:
\[ 1 = |CP(X_i) - CP(X_j)| \le K|X_i - X_j| \le Kh,\]
and therefore
\[ K \ge \frac{1}{h}.\]
\end{proof}

Critical points and gradients serve as strategies for gathering information from the post-bottleneck phase. Our analysis above shows two key properties:

\begin{enumerate}
    \item For PointNet \cite{pointnet}, we have at least $N-F$ points with \emph{zero gradients} (those outside the critical set). For $N \gg F$ this means most of the points. Regarding DGCNN \cite{dgcnn}, we observe a similar relaxed trend.
    
    \item The smoothness of the importance measure induced by critical points is inversely proportional to the sampling resolution. That is, critical points become less smooth as the point cloud is sampled at a finer resolution.
\end{enumerate}

The attributes of smoothness and uniqueness are highly desirable for an effective influence measure. In a thought experiment, consider extracting the most influential input, perhaps a single point from the tip of a cone. It becomes evident that the shape is preserved, and we would expect points in close proximity to the filtered one to exhibit higher influence than those farther away. By iteratively applying this process, we anticipate spatially close points to exert approximately similar influence, resulting in a smooth influence map. Moreover, after filtering influential points, some initially non-influential ones may gain significance, while others remain uninfluential.
Thus, influence should be meaningful, with semantic ordered ranking, even for zero-gradient points.



By probing features in the pre-bottleneck stage, our method assesses a point's \emph{potential} to contribute to classification rather than its actual contribution, given a certain point sampling. We are thus able to rank points, even those with zero actual contribution, resulting in a smoother influence. Furthermore, it enables ranking points by semantic meaning, regardless of the sampling resolution (See \cref{fig:fbi_critical}). In \cref{fig:intro}, we visualize the gradients in PointNet\cite{pointnet} and DGCNN\cite{dgcnn}, highlighting their non-smooth characteristics and the very low influence of parts of the shape. We demonstrate that FBI remains smooth and ranks even less influential parts, as exemplified in \cref{subfig:intro_pointnet}. This approach remains effective for architectures that incorporate learning using neighbors in the featurizing step and employ MeanPooling along with Max-Pooling, such as DGCNN\cite{dgcnn}, as demonstrated in \cref{subfig:intro_dgcnn}.

\begin{figure}[ptbh!]
    \centering
    \captionsetup[subfigure]{justification=centering}
    \begin{subfigure}[t]{0.49\textwidth} 
\includegraphics[width=1\textwidth]{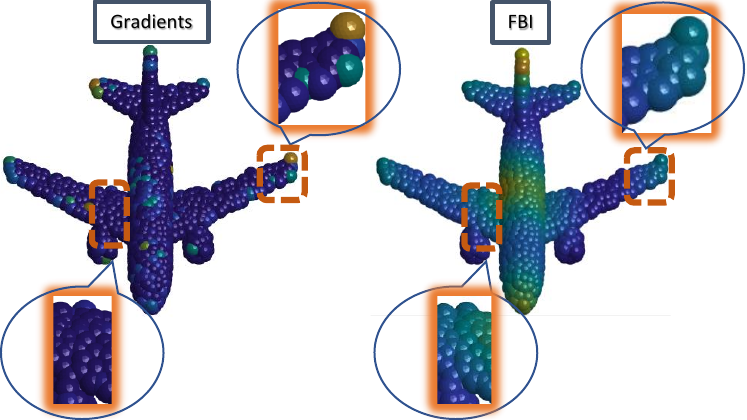}
\caption{{\bf PointNet\cite{pointnet}}}
        \label{subfig:intro_pointnet}
    \end{subfigure}
    \begin{subfigure}[t]{0.49\textwidth} 
\includegraphics[width=1\textwidth]{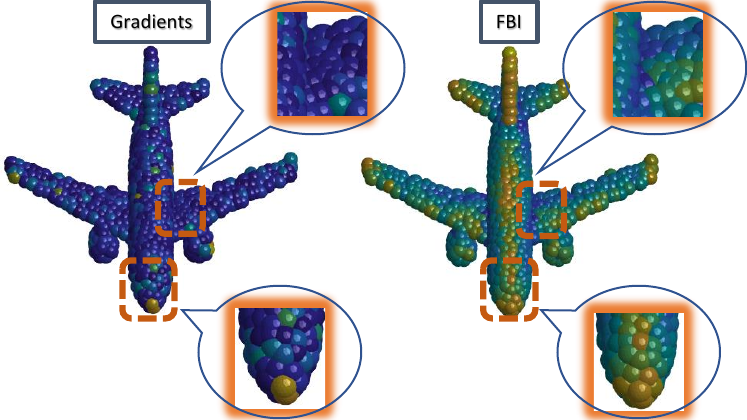}
\caption{{\bf DGCNN\cite{dgcnn}}}    
        \label{subfig:intro_dgcnn}
    \end{subfigure}
    \caption{\textbf{Non-Smooth Gradients.} Gradients in PointNet \cite{pointnet} are zero outside the critical set (e.g., wing's base), and exhibit a non-smooth characteristic (e.g., wing's edge). This trend is similarly observed in DGCNN \cite{dgcnn}. Our approach results in a smoother influence map, predicting the potential influence even for points with zero gradients.}
    \label{fig:intro}
\end{figure}

\begin{figure*}[ptbh!]
  \centering
   \includegraphics[width = 0.9\linewidth, height =0.415 
   \linewidth]{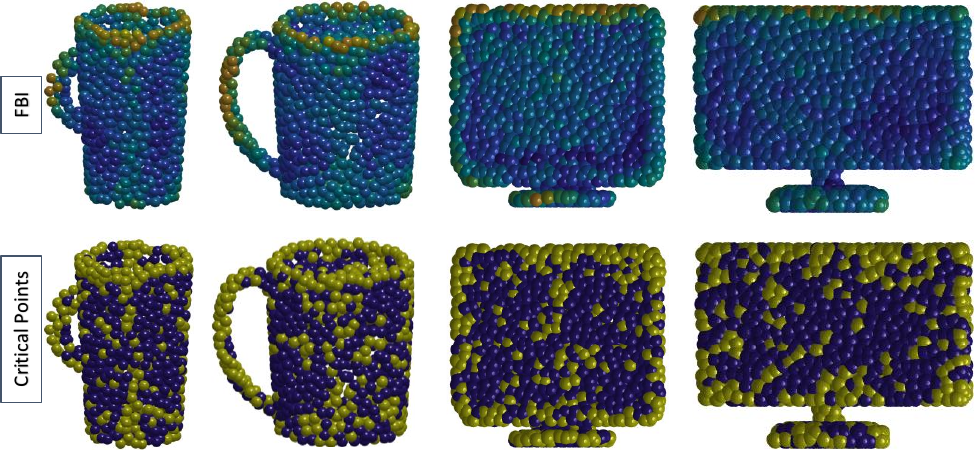}
   \caption{{\bf FBI (ours) vs. Critical Points.} FBI provides rankings based on semantic meaning across the entire shape. Notably, elements like the cup handle or the top of the monitor exhibit high influence, while other parts receive smooth ranking. In contrast, critical points predominantly highlight prominent regions but, in other areas, the selection of points appears nearly random.}
   \label{fig:fbi_critical}
\end{figure*}

\subsection{Method analysis}

\subsubsection{Performance.} Comparisons conducted in \cref{table:auc} show that our method outperform others in most of the networks, with extreme improvement in PointNet\cite{pointnet} on perturbation test. 
 In this test, points are systematically removed (ranging from 10\% to 90\%), starting with the most influential ones. The accuracy is averaged over all 2468 instances in the ModelNet40 set, and the overall test performance is summarized using the area-under-the-curve (AUC). The observed suboptimal performance of gradients and critical points \cite{pointnet} may be attributed to uniformly zero gradients, as when the entire critical set is filtered, non-critical points are randomly perturbed.

\begin{table*}
\begin{center}
  \begin{tabular}{p{3.5cm} || C{2cm} C{2cm} C{2cm} C{2cm}}
  
    \hline
    \multirow{1}{*}{Method} &
      \multicolumn{1}{c}{DGCNN\cite{dgcnn}}&
      \multicolumn{1}{c}{RPC\cite{modelnet_c}}&
      \multicolumn{1}{c}{PointNet\cite{pointnet}}&
      \multicolumn{1}{c}{GDANet\cite{gdanet}}\\
    \hline
    \multirow{1}{*}{Random Sampling} 
    & 55.60 & 66.12 & 68.65 & 59.43\\
    \hdashline
    \multirow{1}{*}{Lime(C=128)\cite{point_lime}} 
    & \textbf{34.80} & \underline{47.22} & \underline{50.68} & \underline{43.52} \\
    \multirow{1}{*}{Lime(C=1024)\cite{point_lime}} 
    & 52.97 & 62.22 & 63.67 & 56.02 \\
    
    \multirow{1}{*}{Gradients} 
    & 50.64 & 59.71 & 61.95 & 54.43\\
    \multirow{1}{*}{IntegratedGradients \cite{axiomatic}} 
    & 41.38 & 56.63 & 59.51 & 48.65\\
    
    \multirow{1}{*}{Critical Points \cite{pointnet}} 
    & 51.66 & 61.93 & 64.08 & 57.85\\
    \hline
    \multirow{1}{*}{FBI \textbf{(Ours)}} 
    & \underline{41.05} & \textbf{43.57} & \textbf{39.20} & \textbf{40.00}\\
    \hline
  \end{tabular}
\end{center}
\caption{\textbf{Perturbation Test (AUC) on ModelNet40\cite{modelnet40}.} Our FBI method outperforms all other baselines on 3 out of 4 examined networks, with the advantage of being 7 orders of magnitude faster than Lime\cite{lime} as the only candidate with competitive results.
}
\label{table:auc}
\end{table*}

\begin{table} [tbh]
\begin{center}
  \begin{tabular}{p{3.5cm} || C{2cm} C{2cm} C{2cm} C{2cm}}
  
    \hline
    \multirow{2}{*}{Method} &
      \multicolumn{4}{c}{Time[ms]}\\
      \cmidrule(l){2-5} &
       PointNet\cite{pointnet} & GDANet\cite{gdanet} & DGCNN\cite{dgcnn} & RPC\cite{modelnet_c}\\
    \hline
    Lime(C=128) \cite{point_lime}& 50,000 & 50,000 & 50,000 & 50,000\\
    Lime(C=1024) \cite{point_lime}& 500 & 750 & 600 & 560\\
    Gradients  & 6 & 15 & 8 & 15\\
    IntegratedGradients \cite{axiomatic}& 40 & 80 & 50 & 65\\
    Critical Points \cite{pointnet} & 0.008 & 0.008 & 0.008 & 0.008\\
    \hline
    FBI \textbf{(ours)} & \textbf{0.003} & \textbf{0.003} & \textbf{0.003} & \textbf{0.003}\\
    \hline
  \end{tabular}
\end{center}
\caption{\textbf{Timing.} Our approach obtains at least three orders of magnitude speedup, compared to modern XAI methods. Critical points is also very fast but much less accurate (\cref{table:auc}). Timing of our method is approximately constant, regardless of the network's architecture, since no derivation across the layers  is performed. The method is thus scalable in terms of network parameters or size of point cloud.}
\label{table:timing}
\end{table}

\subsubsection{Timing.}

FBI measure involve straightforward calculations on features, eliminating the need for time-consuming derivations across the entire network, as seen in Gradients and IntegratedGradients\cite{axiomatic}, or any iterative processes involved in Lime \cite{lime}.
Consequently, our simple method is well-suited for time-demanding processes, particularly when considering the application of explainable methods during inference. Given its purely computational nature, our method exhibits scalability, making it particularly advantageous for larger networks.

\begin{figure*}[ptbh!]
  \centering
   \includegraphics[width = 0.7\linewidth, height =0.385 
   \linewidth]{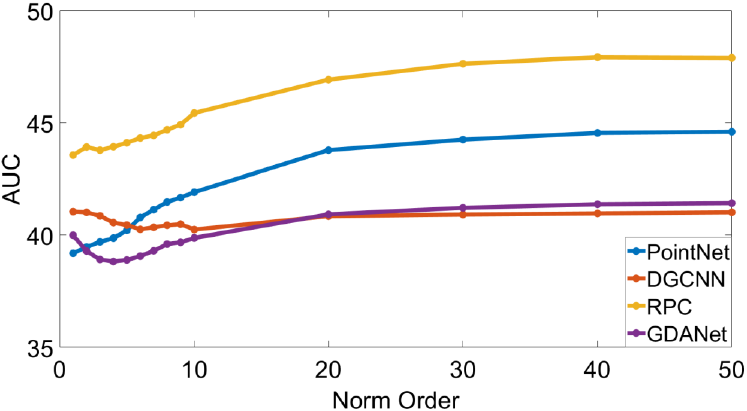}
   \caption{{\bf Norm order comparison.} AUC as a function of the order $p$ of the $L^p$ norm, where the order $p$ measure is $FBI_p(i):=\|X_F(i,\cdot)\|_{L^p}$. Optimal results for RPC \cite{modelnet_c} and PointNet \cite{pointnet} are achieved with the $L^1$ norm, while DGCNN \cite{dgcnn} and GDANet \cite{gdanet} show improved performance with a higher norm order. To maintain simplicity, we adopt the $L^1$ norm in our method.}
   \label{fig:norm_ablation}
\end{figure*}

\subsubsection{Ablation study}
Our proposed method is simple and parameter-free. For completeness, we investigate the impact of different $L^p$ norms to compute FBI. In \cref{fig:norm_ablation}, we assess the AUC on a grid of norm orders $p$ for PointNet\cite{pointnet}, RPC \cite{modelnet_c}, GDANet \cite{gdanet}, and DGCNN\cite{dgcnn}.

\section{Analysis and Insights}

Beyond its conventional role in debugging, explainable AI proves to be a powerful tool for illuminating fundamental aspects of 3D analysis. In this section, we employ FBI to gain a comprehensive understanding of key facets. Specifically, we begin by presenting and comparing the influence maps of rotation-invariant networks against their classic counterparts. Providing insights into the intricate decision-making processes of the network when confronted with out-of-distribution scenarios, we then discuss distinctions between self-supervised and supervised method.


\begin{figure*}[ptbh!]
  \centering
   \includegraphics[width = 1\linewidth, height =0.54 
   \linewidth]{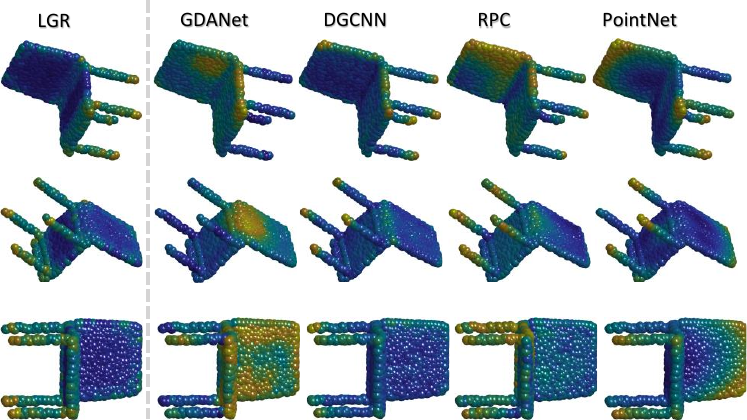}
    \caption{\textbf{Illustration of rotation invariance.} A chair at different rotations, color-coded by our FBI measure, highlighting the influence distribution across various orientations. The Local-Global-Representation (LGR) network, designed for rotation invariance, exhibits more consistent influence distribution compared to traditional networks.}
   \label{fig:rotation_invariant}
\end{figure*}

\subsection{3D rotation invariance}
A crucial aspect of 3D classification involves accounting for object rotations to ensure that a rotated object is consistently classified as the same object. This fundamental characteristic has spurred the development of rotation-invariant classification networks by researchers. One notable example is the Local-Global-Representation (LGR) network \cite{lgrnet}, designed to seamlessly integrate local geometry and global topology in a rotation-invariant manner.
In \cref{fig:rotation_invariant}, we present an example of a chair with different rotations. As can be expected, the influence distributed on the rotated shapes appears more consistent across various rotations in LGR \cite{lgrnet}, highlighting its effectiveness as a rotation-invariant network. 
In contrast, traditional networks are notably affected by the rotation of the shape, with influence distributed differently over the shape for each rotation.

\subsubsection{Quantitative analysis}
To rigorously validate our observations, we conduct a quantitative analysis to assess the impact of rotations on various networks. For a rotation-invariant network, we anticipate consistent influence for each point irrespective of the rotation of the shape.
To quantify the influence deviation of rotated shapes, we compute a per-point deviation measure represented as $\delta = \frac{||FBI^{rotated} - FBI^{clean}||}{FBI^{clean}}$. Here, $FBI^{rotated} \in \mathbb{R}^N$ is computed on the rotated shape, and $FBI^{clean} \in \mathbb{R}^N$ is the influence measure of the unrotated shape. This deviation measure is averaged across all points of the shape, all shapes in the dataset, and across all severities of rotations. It effectively gauges the extent of feature magnitude deviation induced by rotation compared to the clean feature magnitude.
In \cref{table:rotation_invariance_quantitatively}, we present a summary of the correlation between $\delta$ and accuracy under rotations. A network that tends to maintain consistent influence for each point during rotations is better equipped to handle rotational variations. It is noteworthy that even LGR-Net\cite{lgrnet}, designed for rotation invariance, does not perfectly preserve influence under rotations.
\begin{table} [ptbh!]
\begin{center}
  \begin{tabular}{p{2cm} || C{1.5cm} C{2.0cm}}
  
    \hline
    Model & $\delta$[\%] & Accuracy [\%]\\
    \hline
    LGR \cite{lgrnet} & 1\% & 91.1\%\\
    GDANet \cite{gdanet}  & 52\% & 78.8\%\\
    DGCNN \cite{dgcnn} & 174\% & 78.5\%\\
    RPC \cite{modelnet_c}  & 215\% & 76.8\%\\
    PointNet \cite{pointnet}  & 2873\% & 59.1\%\\
    \hline
  \end{tabular}
\end{center}
\caption{\textbf{Quantitative analysis of rotation invariance.} Feature deviation percentage ($\delta$) and classification accuracy for various 3D classification models under rotations. Lower $\delta$ values indicate consistent feature influence during rotations, this is well correlated to higher accuracy.}

\label{table:rotation_invariance_quantitatively}
\end{table}

\subsection{Robustness to out-of-distribution (OOD)}

\begin{figure*}[ptbh!]
  \centering
  \includegraphics[width=1\linewidth, height=0.4285\linewidth]{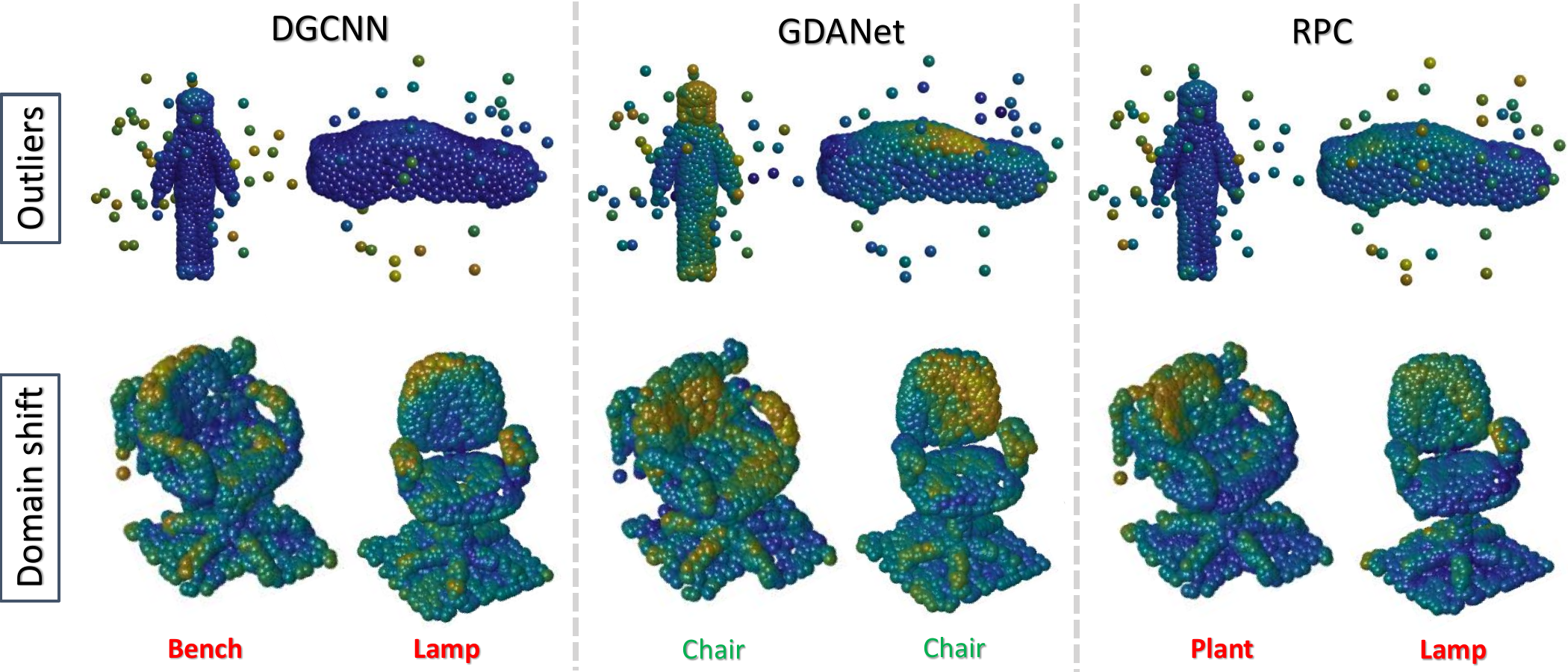}
  \caption{{\bf OOD robustness analysis.} Color-coded by FBI. Networks trained on ModelNet40 \cite{modelnet40}, and evaluated either on corrupted ModelNet-C \cite{modelnet_c} or real-world ScanObjectNN \cite{scanobjectnn}. Visualization demonstrate that the influence corresponds to the ability of networks to handle OOD. Architectures (such as GDANet) influenced by semantic regions, even in the presence of outliers or background, are more OOD robust.}
  \label{fig:ood}
\end{figure*}

\subsubsection{Outliers}
In \cite{agnostophobia}, the authors argued and visualized that the feature magnitudes of unknown samples are lower than those of known ones in image classification. In this context, we investigate the same characteristic on point clouds and surprisingly observe the opposite trend. Outlier points exhibit higher feature magnitudes than benign points.

This observation holds across multiple architectures trained on ModelNet40\cite{modelnet40} specifically, DGCNN\cite{dgcnn}, GDANet\cite{gdanet}, and RPC\cite{modelnet_c}. To visualize this phenomenon, we use FBI to examine the influence maps of these networks on corrupted samples from ModelNet-C\cite{modelnet_c}, focusing on Add-Global corruption. The networks were trained on uncorrupted samples, without outliers, and evaluated on corrupted ones. Therefore, outliers introduced in ModelNet-C, are categorized as OOD, since they were not introduced during training.
The visualization in \cref{fig:ood} illustrates that in 3D classification, outliers tend to be highly influential. Consequently, the magnitude of OOD features is higher than that of in-distribution features.


To quantitatively validate this assertion, we compute the attention gained by outliers relative to the total influence distributed over the entire shape. Let $i$ denote a sample index,  ${O_i}$ be the outlier points set and $S_i$ be the set of all points in the shape i.e, $O_i \subset S_i$. We define $R_i$ as the fractional influence that outliers gained by:
\[
R_i = \frac{\Sigma_{j \in O_i} FBI(j)}{\Sigma_{j \in S_i} FBI(j)}.
\]
We average $R_i$ over the entire add-global set across all degrees of severity. In \cref{fig:ood_graph}, we demonstrate that as the network tends to allocate more influence to the outliers, the overall performance drops.

\begin{figure}%
  \centering
  \subfloat[][]{%
      \includegraphics[width=1\textwidth]{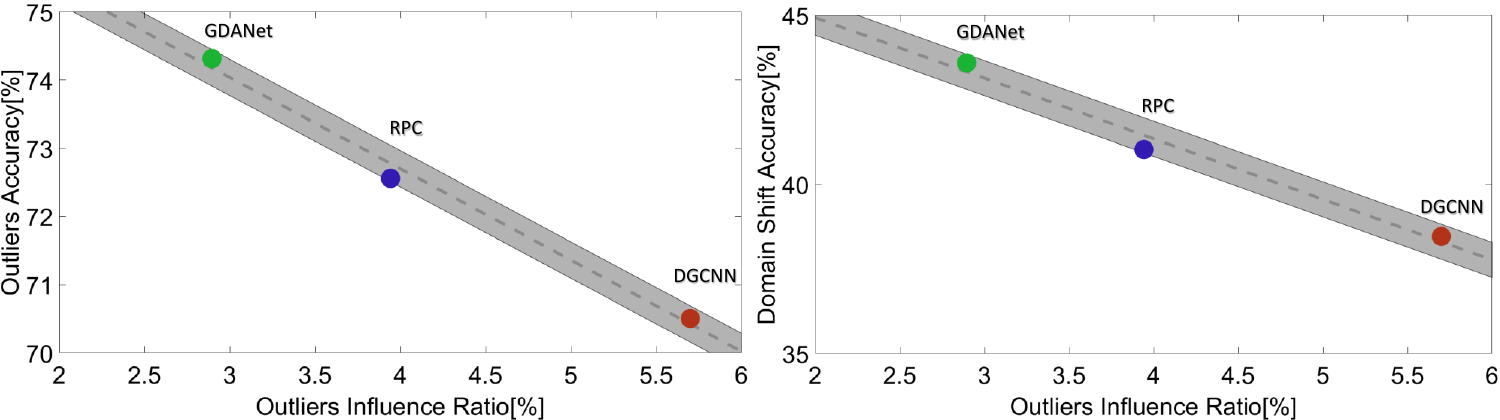}}%
  \qquad
  \subfloat[][]{\begin{tabular}{p{2cm} || C{1cm} C{2cm} C{2cm}}
  
    \hline
    Model & R  & Outliers Accuracy & Domain Shift Accuracy \\
    \hline
    DGCNN \cite{dgcnn} & 5.70\% & 70.50\% & 38.46\%\\
    RPC \cite{modelnet_c}  & 3.94\% & 72.55\% & 41.03\%\\
    GDANet \cite{gdanet}  & 2.89\% & 74.31\% & 43.59\%\\
    \hline
  \end{tabular}}
  \caption{\textbf{Correlation between R and OOD Accuracy.} A linear dependency is observed between the fraction of influence (R) allocated to outliers and OOD robustness. Networks are trained on ModelNet40 \cite{modelnet40} and evaluated on ModelNet-C\cite{modelnet_c} (outliers) as well as on ScanObjectNN\cite{scanobjectnn} (representing domain shift to real-world).}%
  \label{fig:ood_graph}%
\end{figure}

\pagebreak
\subsubsection{Domain shift}

Another crucial aspect of OOD evaluation involves training on one domain and assessing performance on another. In this scenario, we focus on networks trained on the synthetic dataset, ModelNet40 \cite{modelnet40}, and evaluate their performance on the more challenging ScanObjectNN\cite{scanobjectnn} dataset. The latter encompasses real-world point clouds often affected by challenging conditions, including outliers and complex backgrounds.
We investigate how measuring the fractional influence can indicate efficiency for domain shift scenarios.  

Similarly to the previously observed trend, \cref{fig:ood} (bottom) illustrates the ability of GDANet to grasp relevant shape details in the presence of real-world challenges, making it well-suited for domain shift tasks, compared to other examined networks. To quantitatively evaluate this insight, we assess accuracy on the Chair class, a category shared between ModelNet40 and ScanObjectNN. The results in \cref{fig:ood_graph} support the consistent trend, where the resilience of GDANet \cite{gdanet} to outliers aligns with its effectiveness in handling domain shifts, outperforming RPC \cite{modelnet_c} and DGCNN \cite{dgcnn}.

\subsection{Supervised and self-supervised learning analysis}
In image classification, prior studies have illustrated distinctions in influence maps derived from both supervised and self-supervised paradigms, even when applied to identical architectural configurations. In the case of the Vision Transformer (VIT) \cite{vit}, trained in a supervised manner on Imagenet \cite{imagenet}, the acquired influence maps manifest a tendency to attend to features not directly associated with the predicted object. For instance, an image featuring a cow surrounded by grass generates an influence map attending both the cow and the surrounding grass. This problematic phenomenon, denoted as \textit{shortcuts} or \textit{spurious cues} \cite{shortcuts, cues}, is attributed to dataset bias. The training dataset predominantly features instances of cows against a grassy backdrop, leading the classifier to erroneously associate the presence of the cow with the concurrent existence of a grassy background. 
In contrast, influence maps derived from Dino-VIT \cite{dino_vit}, a Vision Transformer architecture trained under a self-supervised regime, exhibit a greater concentration on the predicted object.
To the best of our knowledge, our study represents a pioneering effort in analyzing the influence maps of 3D classification networks within this specific context, illuminating analogous insights.

\begin{figure}[ptbh!]
  \centering
   \includegraphics[width=1\linewidth, height=0.57\linewidth]{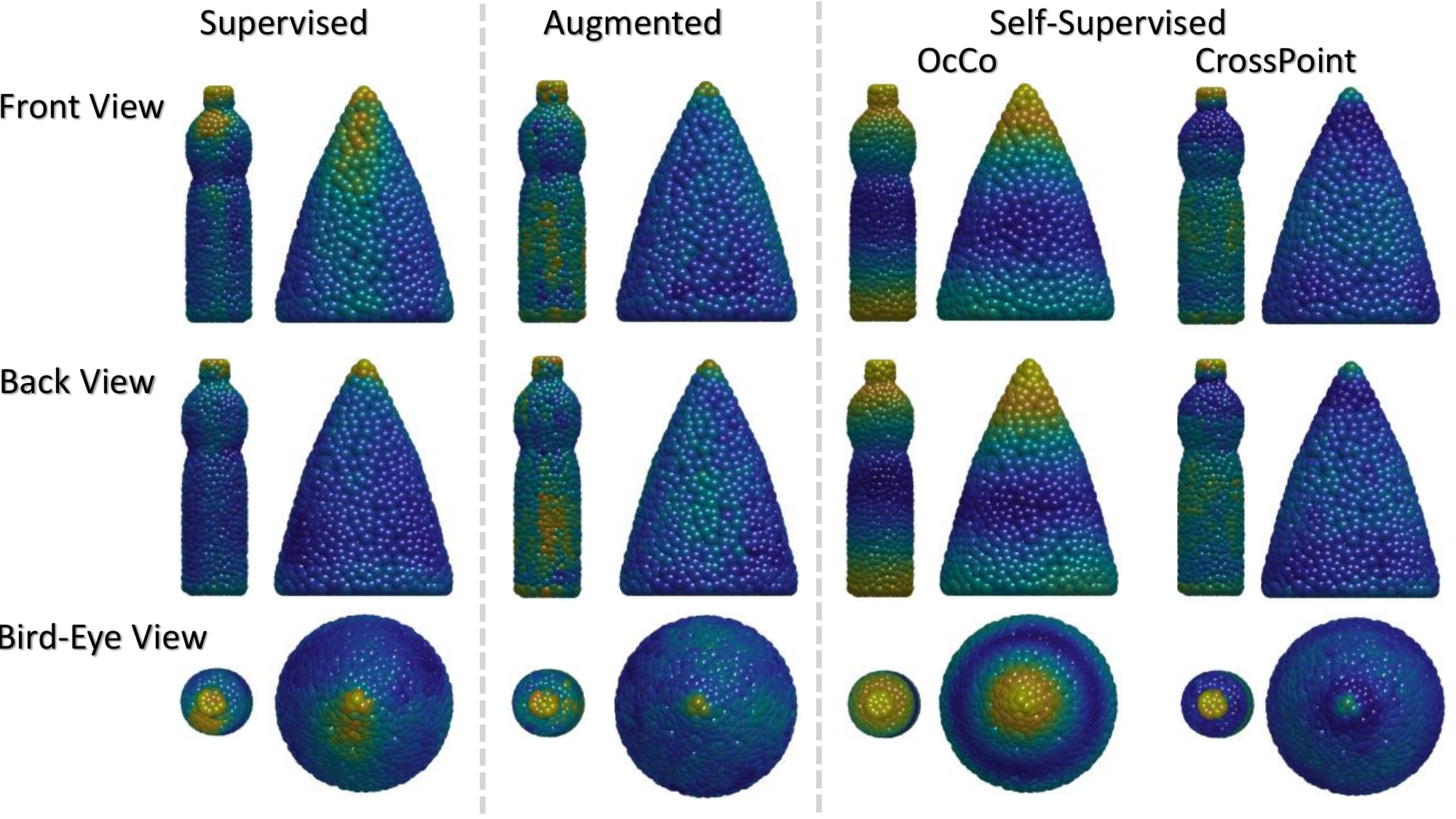}
   \caption{{\bf Influence on supervised and self-supervised methods.} All methods utilize DGCNN \cite{dgcnn} as a backbone. The supervised approach exhibits asymmetric influence, emphasizing the frontal aspect despite the symmetry of the shape (a bottle and a cone). In both OcCo\cite{occo} and CrossPoint\cite{cross_point}, the influence is symmetric, suggesting a potential dataset bias. A WolfMix\cite{pointwolf, rsmix} augmented version slightly alleviates the asymmetry but heavily depends on the augmentation procedure.}
   \label{fig:supervised_self_supervised_modelnet40}
\end{figure}

\subsubsection{Dataset bias.}

we explore the influence maps produced by two self-supervised methods:

1) \textit{CrossPoint} \cite{cross_point}: This method learns 3D features by minimizing contrastive loss on image-to-point-cloud correspondences.

2) \textit{Occlusion Completion (OcCo)} \cite{occo}: This approach focuses on reconstructing obscured regions from a camera view.

To ensure a fair comparison, we employ DGCNN \cite{dgcnn} as the backbone framework for all methods.
Unraveling spurious cues in ModelNet40 \cite{modelnet40} presents a challenge as all data points are inherent to the object itself, lacking a distinct background. However, for objects characterized by symmetry, we anticipate a corresponding symmetrical influence. In \cref{fig:supervised_self_supervised_modelnet40}, we compare influence maps for a bottle and a cone featuring z-axis symmetry. Evaluation of the influence map from a DGCNN trained in a supervised fashion reveals a bias toward the frontal region of the object, resulting in an asymmetric influence on a symmetric shape. Interestingly, with OcCo \cite{occo} and CrossPoint \cite{cross_point}, the influence measure exhibit impressive symmetry. The influence, when cultivated through a self-supervised approach, aligns more symmetrically with the inherent symmetry of the object. We further analyze the effect of augmentation \cite{pointwolf, rsmix} on DGCNN. One can observe  symmetry is increased, however, remains of asymmetry are still present, since this approach may depend on the augmentation procedure.

This phenomenon could be attributed to dataset bias. For instance, if the majority of cups in the dataset have handles positioned at the frontal aspect, the network might disproportionately focus on this region in its pursuit of discriminative elements. In a self-supervised setting, where labels are absent, there is a potential reduction in susceptibility to such biases.

\section{Discussion and Conclusion}
In this paper, we introduced a fast and simple  explainability method for point cloud data. 
Such measures should not be sensitive to small sampling perturbations. Some local smoothness is required along with sufficient variance also in less important regions. Our analysis shows that the Max-Pooling bottleneck, common in graph networks, induces distinct gradient features. The gradients have extreme values within some highly distinct regions and are close to zero elsewhere. Thus gradients yield non-smooth data on one hand and flat regions on the other hand. The performance of gradient-based importance methods is hence degraded. Perturbation-based methods, such as \cite{point_lime}, perform better, but are highly intensive computationally (about 6 orders of magnitude slower than our method). We suggest to use pre-bottleneck data (before Max-Pooling) and specifically show that the $L^1$ norm of the features (per-point) is a high quality importance indicator. SOTA results are achieved based on this measure.
We show how this information can be used to analyze network performance, including invariance analysis, effects of augmentation and self-supervised learning and the ability to handle outliers and data shifts.



\bibliographystyle{splncs04}
\bibliography{main}
\pagebreak
\title{Supplementary Materials - Fast and Simple Explainability for Point Cloud Networks} 

\titlerunning{Fast and Simple
Explainability for Point Cloud Networks (Supp)}

\author{Meir Yossef Levi\inst{1} \and
Guy Gilboa\inst{1}}


\institute{Technion - Israel Institute of Technology, Haifa, Israel \\
\email{me.levi@campus.technion.ac.il}\\
\email{guy.gilboa@ee.technion.ac.il}}

\maketitle


\begin{lemma}[Auxiliary]
Assume a K-nearest-neighbors (KNN) graph constructed by $X$ is a connected graph, and assume that $S_c, \bar{S}_c \notin \emptyset$, then $\exists i,j: CP(X_i) \neq CP(X_j)$ and $X_j \in \text{KNN}(X_i)$.
\label{lemma:Auxiliary}
\end{lemma}

\begin{proof}
Assume, in contradiction, that $\forall i,j: X_j \in \text{KNN}(X_i)$, $CP(X_j) = CP(X_i)$ holds. Since the graph is connected, then each $X_j$ is a k-hop neighbor of $X_i$ for some $k$. Without loss of generality, assume that $CP(X_i) = 1$ for some $i$. Then on each hop, following the assumption, $CP(X_j) = 1, \forall j:X_j \in \text{KNN}(X_i)$. Keep propagating to all nodes in the graph, and we get that $CP(X_i) = 1$ for all $i$ in the graph. Then we get that $\bar{S}_c \in \emptyset$ in contradiction to the assumption that $S_c, \bar{S}_c \notin \emptyset$.
\end{proof}

\section{Additional Visualizations}

The visualizations presented below serve as extensions to the demonstrations discussed in the paper. They aim to enhance conceptual understanding and highlight the broad applicability of the insights presented in this paper.
For convenience, we provide here the results oulined in the main paper regarding rotation-invariance and OOD accuracy (\cref{table:Results_reminder}).

\begin{table} [ptbh!]
\begin{center}
  \begin{tabular}{p{2cm} || C{2.0cm} C{2.0cm} C{2.0cm}}
  
    \hline
    Model & Rotation-invariance accuracy & Outliers accuracy & Domain shift accuracy \\
    \hline
    LGR & 91.1\% & -- & --\\
    GDANet  & 78.8\% & 74.31 \% & 43.59 \%\\
    DGCNN & 78.5\% & 70.50 \% & 38.46 \%\\
    RPC  & 76.8\% & 72.55 \% & 41.03 \%\\
    PointNet  & 59.1\% & -- & --\\
    \hline
  \end{tabular}
\end{center}
\caption{\textbf{Reminder for rotation-invariance and OOD accuracy results.}}

\label{table:Results_reminder}
\end{table}

\begin{figure*}[ptbh!]
  \centering
   \includegraphics[width=0.9\linewidth, height=1.33\linewidth]{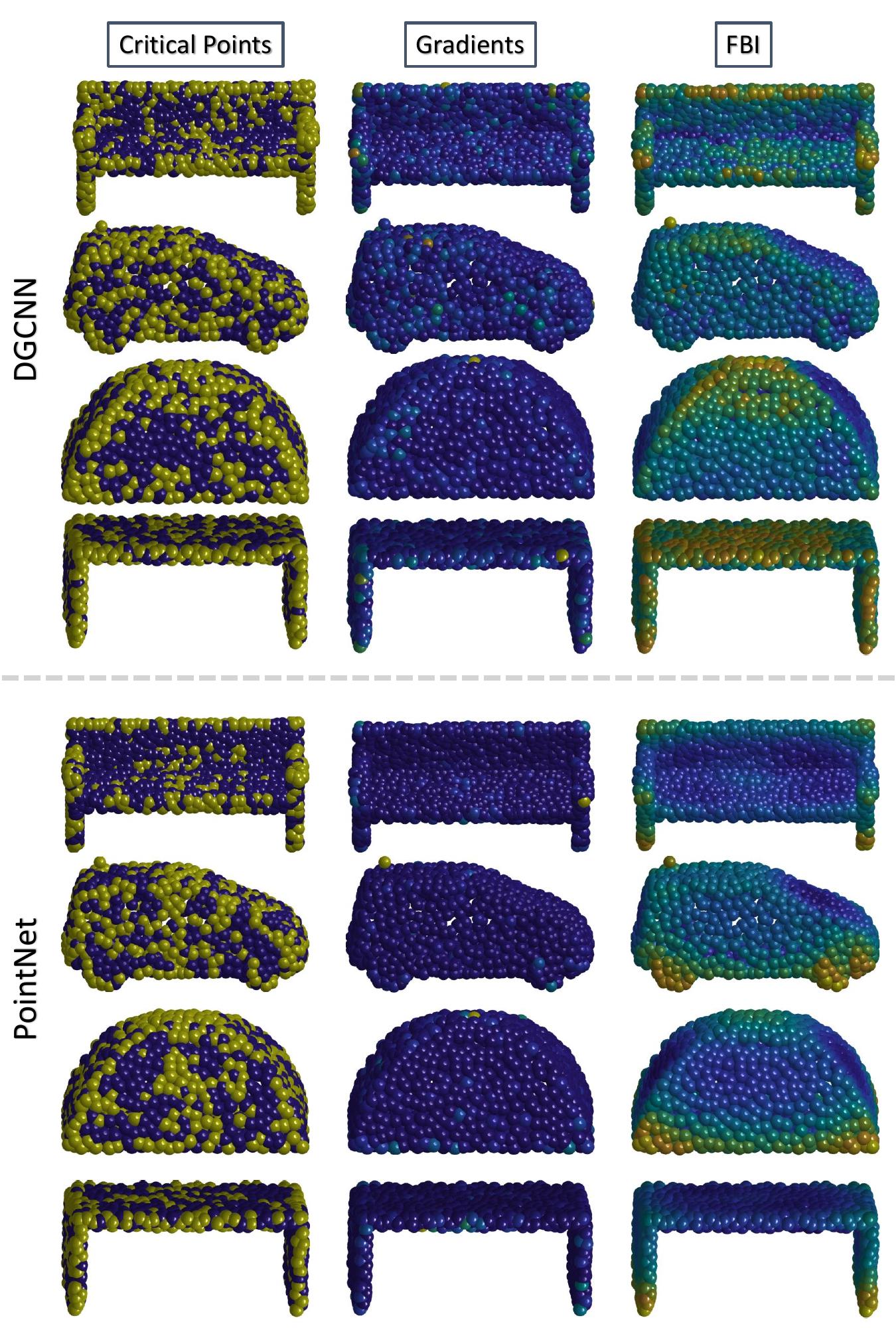}
   \caption{{\bf Critical Points,  Gradients and FBI.} Additional samples are provided to illustrate the smoothness of our method. Both critical points and gradients exhibit a substantial number of points with low or zero influence.}
   \label{fig:supp_gradients}
\end{figure*}

\begin{figure*}[ptbh!]
  \centering
   \includegraphics[width=1\linewidth, height=0.375\linewidth]{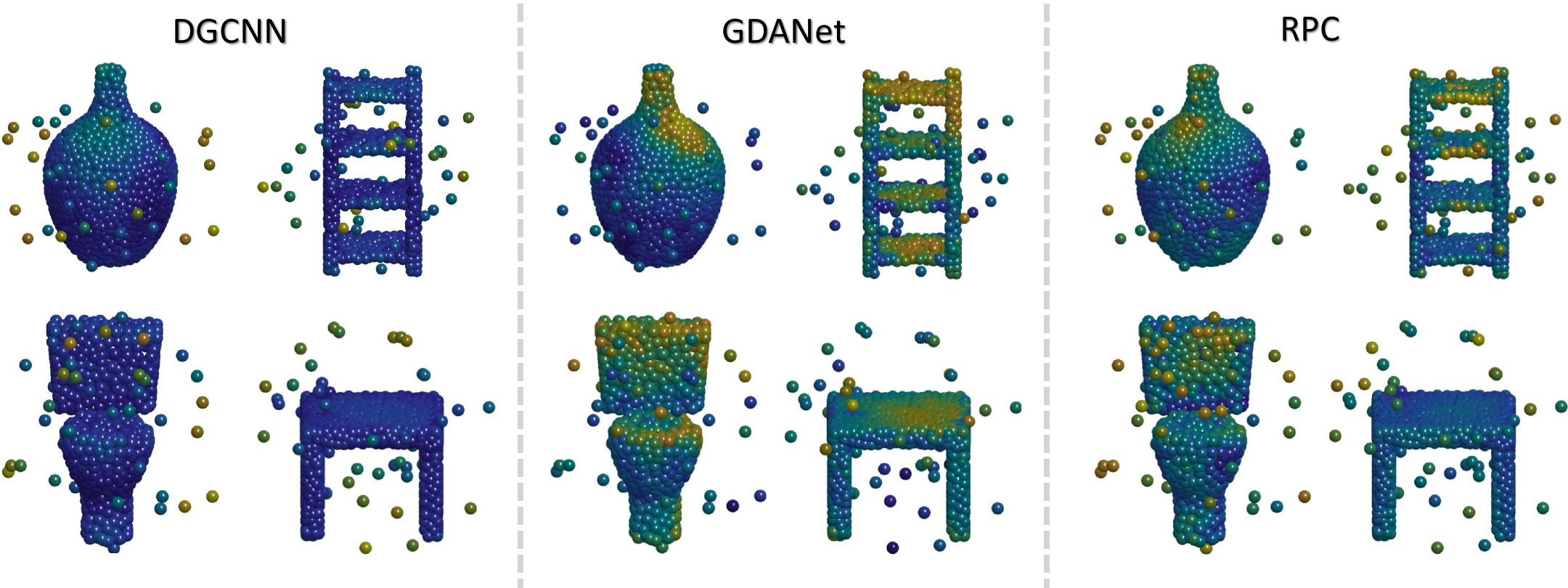}
   \caption{{\bf OOD evaluation on additional samples.} Validations of the observed trend that GDANet exhibits a higher capability in attending to semantic meaning rather than outliers, whereas DGCNN performs comparatively worse.}
   \label{fig:supp_ood}
\end{figure*}

\begin{figure*}[ptbh!]
  \centering
   \includegraphics[width = 1\linewidth, height =0.56 
   \linewidth]{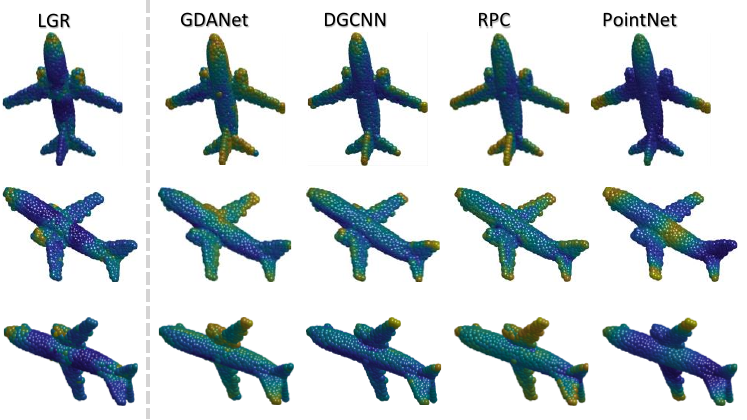}
   \caption{{\bf Rotation invariance.} Airplane example.}
   \label{fig:supp_rotation_invariant_airplane}
\end{figure*}

\begin{figure*}[ptbh!]
  \centering
   \includegraphics[width = 1\linewidth, height =0.56 
   \linewidth]{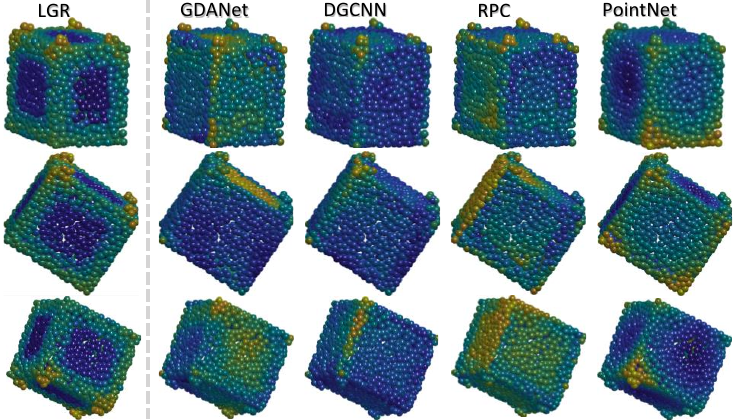}
   \caption{{\bf Rotation invariance.} TV-Stand example.}
   \label{fig:supp_rotation_invariant_box}
\end{figure*}

\begin{figure*}[ptbh!]
  \centering
   \includegraphics[width=0.85\linewidth, height=1.5107\linewidth]{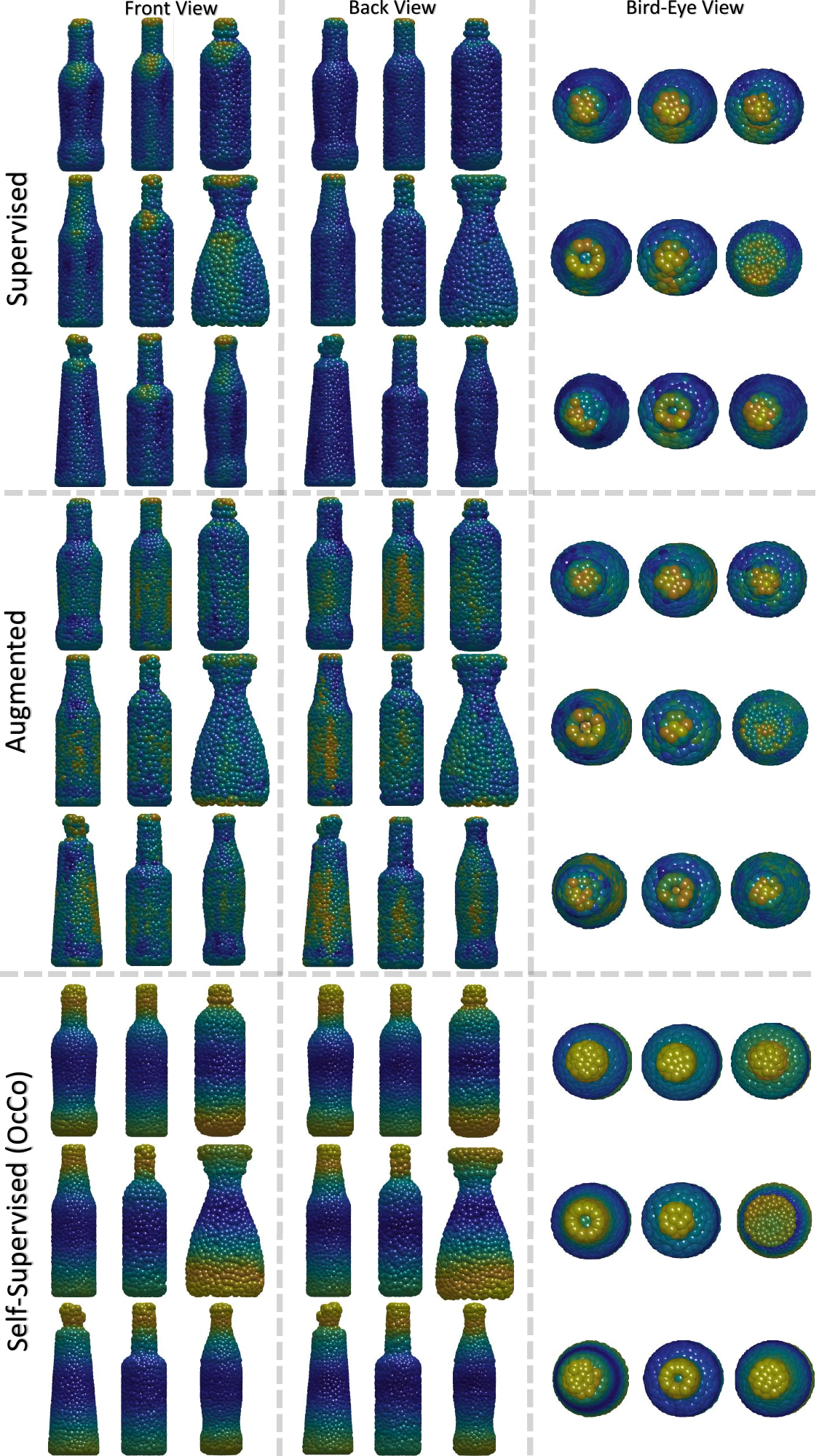}
   \caption{{\bf Emphasis on frontal aspect.} Additional examples of bottles emphasize the frontal aspect. The self-supervised method demonstrates almost perfect symmetric influence, while augmented samples fall in between.}
   \label{fig:supp_supervised_self_supervised_modelnet40}
\end{figure*}

\end{document}